\documentclass{article}

\usepackage{amsmath,cite}
\usepackage{amssymb}
\usepackage{amsthm}
\usepackage{lipsum,comment}
\usepackage{amsfonts}
\usepackage{lipsum}
\usepackage{amsfonts}
\usepackage{graphicx}
\usepackage{epstopdf}
\usepackage{algorithm}
\usepackage{algorithmic}
\usepackage{epsfig,amsmath,amssymb}
\usepackage{color}
\usepackage{caption}
\usepackage{subfigure,hyperref}
\usepackage{url}
\usepackage{bm}
\usepackage{dsfont,comment}
\usepackage[normalem]{ulem}

\usepackage[usenames,dvipsnames,svgnames,table]{xcolor}

\newtheorem{theorem}{Theorem}[section]

\newtheorem{proposition}[theorem]{Proposition}

\newtheorem{definition}[theorem]{Definition}

\newlength{\IntHtx}
\newcommand{\Intvalx}[2]{\mspace{-1mu}\left.\rule[0pt]{0pt}{\IntHtx}\right|_{\,#1}^{\,#2}}

\begin{document}

\title{Convexity Shape Prior for Level Set based Image Segmentation Method}%

\author{Shi~Yan, Xue-cheng~Tai, Jun~Liu, Hai-yang~Huang%
\thanks{Shi~Yan, Jun~Liu, Haiyang~Huang are with Laboratory of Mathematics and Complex Systems (Ministry of Education of China), School of Mathematical Sciences, Beijing Normal University, Beijing, 100875, People's Republic of China, Xue-cheng~Tai is with Department of Mathematics, Hong Kong Baptist University, Kowloon Tong, Hong Kong, email: \url{ysicesword@mail.bnu.edu.cn,xuechengtai@hkbu.edu.hk,jliu@bnu.edu.cn,hhywsg@bnu.edu.cn}}}%

\maketitle

\begin{abstract}
We propose a geometric convexity shape prior preservation method for variational level set based image segmentation methods. Our method is built upon the fact that the level set of a convex signed distanced function must be convex. This property enables us to transfer a complicated geometrical convexity prior into a simple inequality constraint on the function. An active set based Gauss-Seidel iteration is used to handle this constrained minimization problem to get an efficient algorithm. We apply our method to
region and edge based level set segmentation models including Chan-Vese (CV) model with guarantee that the segmented region will be convex.
Experimental results show the effectiveness and quality of the proposed model and algorithm.
\end{abstract}


\maketitle

\section{Introduction}\label{sec:introduction}
Shape prior plays an important role in image segmentation, in which one of the most important is the convexity shape prior. A convexity shape or called a convexity region means that the points inside the region form a convex set. Things with convexity shape are quite common in our daily lives, such as a road sign, a book, or a basketball, \emph{etc.}. Here, we propose a general convexity shape prior preservation method based on level set method. This method can be extended to many level set based image segmentation methods to get some convexity segmentation results.

Shape prior is widely studied in image segmentation. Leventon \emph{et al.} \cite{leventon2000statistical} used statistical methods in shape prior methods. Rousson and Paragios \cite{rousson2002shape} first applied level set method in shape prior methods. Chan and Zhu \cite{chan2005level} purposed a theory which allows the shape prior to translate, scale and rotate based on level set method. Thiruvenkadam \emph{et al.} \cite{thiruvenkadam2007segmentation} could segment multi objects from one single image, even the objects are duplicated or covered. Vu and Manjunath \cite{vu2008shape} combined shape prior methods with graph-cut method \cite{boykov2001fast,kolmogorov2004energy}. Guo \emph{et al.} \cite{guo2015automatic} built a model which allows to choose a best shape prior from multi shape priors which are given by the user. The above works are all about some special shape priors, but in the nature, one thing may have several shapes, i.e., a general shape, such as convexity shape prior.

As for convexity shape prior, Liu \emph{et al.} \cite{liu1999role} studied the importance of convexity shape prior in image, as well as more studies in \cite{gorelick2014convexity,royer2016convexity,strekalovskiy2011generalized,bae2017augmented}. The key of handling convexity shape prior is to describe convexity. Gorelick \emph{et al.} \cite{gorelick2014convexity} introduced some constraints on all the straight lines in the discrete image plane under graph theory, which ensures such a property: for any three points that lie on a straight line, if the outer two points were labeled as the object, then the middle point must also be labeled as the object, which is equal to the definition of convexity shape. Royer \emph{et al.} \cite{royer2016convexity} based on graph theory, defined the path and the straight line path between object points, to ensure that all the points which are on the straight path between two object points must be labeled with object. This property could also lead to a convexity shape. These two methods are based on graph theory which studied the connections between vertexes. However, such a graph based discrete methods could not be easily applied to variational PDE method since the geometrical constraints cannot be easily addressed in variational methods.

On variational PDE method, existing methods on convexity shape prior include \cite{strekalovskiy2011generalized,bae2017augmented}. Strekalovskiy and Cremers \cite{strekalovskiy2011generalized} introduced an energy functional, which on each point defines the penalty of changing label in each direction. With the help of N auxiliary sub-region, this model could successfully segment a N-polygon. For example, if we want to get a circle as a segmentation result, N would be quite large, which requires many auxiliary sub-regions. Bae \emph{et al.} used a relation between convexity and curvature in \cite{bae2017augmented}. They transferred convexity into a constraint on curvature, and proposed an Euler's elastic based segmentation model. To optimize this problem, one needs to solve a nonlinear fourth order PDE.
It would be time-consuming, although some recently developed operator splitting methods can be employed. Besides, it is not easy to strictly keep the constraint numerically,
and thus some of the control parameters would play a very import role on segmentation results.

In this paper, we would like to propose a simple linear constraint for convexity shape prior. It is derived from an import property of a smooth signed distance function (SDF).

SDF is used in the reinitialization step of the level set function (LSF) evolution based on level set method (LSM) \cite{osher1988fronts}. \cite{kass1988snakes,Caselles1993,caselles1997geodesic,Jr1997A} proposed active contour based image segmentation models which took advantage of LSM. Chan and Vese \cite{chan2001active} combine LSM with region based image segmentation method. Rada and Chen \cite{Lavdie2012A} used two level sets to achieve selective image segmentation. Compared with \cite{Lavdie2012A}, Rada and Chen \cite{Rada2013Improved} achieved the same purpose using only one level set. Ali \emph{et al.} \cite{Ali2017Multiphase} can segment multi-objects by only one level set. Although LSM is widely used, the LSF may develop irregularities during the level set evolution. Usually, a reinitialization of LSF is needed \cite{chan2001active} during the curve evolutions. Exactly finding a SDF equals to solve the well-known Eikonal equation. For high computational efficiency,
fast matching method \cite{Sethian2000Level} which is similar to Dijkstra's algorithm \cite{Dijkstra1959A} in graph can always be used to reinitialize SDF. Fast sweeping method \cite{Zhao2005A} which uses Gauss-Seidel iteration with alternating sweeping ordering method is another alternative method for reinitialization task.
To avoid direct reinitialization, Li \emph{et al.} \cite{li2010distance} proposed a penalty method by adding a nonconvex functional in the segmentation energy for the constraint $|\nabla \phi|=1$.
In papers \cite{6944678,10.1007/978-3-642-40811-3_60,doi:10.1118/1.4947126}, they extended this penalty method and proposed a two layers level set approach, the segmentation results are very similar to convexity shape. However, they could not give a theoretical condition to ensure convexity.
In our numerical experiments, we will compare our new segmentation model with these edge and region based level set methods including the well-known Chan-Vese model.

In this paper,
we show that the convex shape prior can be guaranteed by a simple constraint on the signed distance function. This constraint can be easily incorporated into variational models with LSM and solved efficiently by some constrained minimization techniques.

The main contributions of this paper are very clear: firstly, we propose a simple linear inequality constraint for convex shapes prior based on
signed distance functions. Secondly, an efficient active set based Gauss-Seidel algorithm is proposed to numerically handle this inequality constraint. Numerical experiments on real and synthetic images are supplied to show the effectiveness and quality of the proposed model and algorithm.

The rest of our paper is organized as follows: firstly, in section \ref{sec:convex_relatedworks}, some related works on level set segmentation and convex shapes prior are introduced. Nextly, a sufficient condition for convex shape prior together with some mathematical analysis are given in section \ref{sec:mainproof}. In the following section \ref{sec:applyonDRLSE}, we apply our method on the level set based models. Section \ref{sec:solveconvex} contains an algorithm that can strictly keep the proposed inequality constraint. Experimental results and some comparisons are shown in section \ref{sec:exp}. Finally, we conclude this paper in section \ref{convex:conclusion}.

\section{Related works}\label{sec:convex_relatedworks}

In this section, firstly, we will introduce the popular Chan-Vese  model. Then secondly, some convex shape prior methods would be given.

\subsection{Chan-Vese Model}
Chan-Vese model is proposed by Chan and Vese in \cite{chan2001active}, which is an image segmentation model based on level set method and energy minimization method.

The first part of the energy functional is called data term
\begin{equation}\label{cv_dataterm}
\begin{array}{rcl}
F_{1}(c_1,c_2,\Gamma)
&=& \lambda_{1}\displaystyle\biggl.\int_{inside(\Gamma)}|I(x)-c_{1}|^2dx \\
&&+ \lambda_{2}\displaystyle\biggl.\int_{outside(\Gamma)}|I(x)-c_{2}|^2dx.
\end{array}
\end{equation}
Here $\Gamma$ is an arbitrary closed curve, $inside(\Gamma)$ denotes the region inside curve $\Gamma$, as well as $outside(\Gamma)$ denotes the region outside curve $\Gamma$, $c_{1}$ and $c_{2}$ are unknown constants, which denote the inner and outer means of the two region $inside(\Gamma)$ and $outside(\Gamma)$, respectively.

The second part of the energy functional is called regularization term, as
\begin{equation}\label{cv_regterm}
F_2(\Gamma) = \mu Length(\Gamma) + \nu Area(inside(\Gamma)),
\end{equation}
where $Length(\Gamma)$ denotes the Length of curve $\Gamma$, $Area(inside(\Gamma))$ denotes the area of region $inside(\Gamma)$.

Combine (\ref{cv_dataterm}) and (\ref{cv_regterm}) together we have,
\begin{equation}
\label{cvmodel}
\begin{array}{rcl}
F_3(c_{1},c_{2},\Gamma) &=& \mu Length(\Gamma) + \nu Area(inside(\Gamma)) \\
&&+ \lambda_{1}\displaystyle\biggl.\int_{inside(\Gamma)}|I(x)-c_{1}|^2dx\\
&& + \lambda_{2}\displaystyle\biggl.\int_{outside(\Gamma)}|I(x)-c_{2}|^2dx.
\end{array}
\end{equation}
While in $\mathbb{R}^{N}$, we have the following equation,
\begin{displaymath}
Area(inside(\Gamma))\leq c (Length(\Gamma))^{\frac{N}{N-1}} .
\end{displaymath}
Thus, we can set $\nu=0$.

By introducing the Heaviside function $H(z)$, Dirac function $\delta(z),z\in \mathbb{R}$ and level set function $\phi(x),x\in\Omega$, then ({\ref{cvmodel}}) can be reformulated as
\begin{equation}\label{convex:cvmodel}
\begin{array}{rcl}
&& F(c_{1},c_{2},\phi) = \mu\displaystyle\biggl.\int_{\Omega} \delta(\phi(x)) |\nabla \phi(x)|dx \\
&&+\lambda_1\displaystyle\biggl.\int_{\Omega}|I(x)-c_{1}|H(\phi(x))dx \\
&&+ \lambda_2 \displaystyle\biggl.\int_{\Omega}|I(x)-c_{1}|(1-H(\phi(x)))dx.
\end{array}
\end{equation}
Here
\begin{displaymath}
H(z) =
\left\{
\begin{array}{ll}
1,& \text{if } z\geq0, \\
0,& \text{if } z< 0,
\end{array}
\right.
\end{displaymath}
and
\begin{displaymath}
\delta(z)= H^{'}(z).
\end{displaymath}
In implementation, $H$ and $\delta$ can be replaced by some smoothed version, see \cite{chan2001active}.
%
%
It is easy to change the length term in the Chan-Vese model by a weighted length associated with an edge detector.
 Let us use the edge detector
 \begin{displaymath}
g(x) = \frac{1}{1+ |\nabla (G * I)(x) |^2}.
\end{displaymath}
Here G is a Gaussian kernel function, $*$ represents the convolution operator, and $\nabla$ represents the gradient operator.
The CV model with an edge detector is then
\begin{equation}\label{eq:DRLSE}
\begin{array}{rcl}
&& F(c_{1},c_{2},\phi) = \mu\displaystyle\biggl.\int_{\Omega} g(x) \delta(\phi(x)) |\nabla \phi(x)|dx \\
&&+\lambda_1\displaystyle\biggl.\int_{\Omega}|I(x)-c_{1}|H(\phi(x))dx \\
&&+ \lambda_2 \displaystyle\biggl.\int_{\Omega}|I(x)-c_{1}|(1-H(\phi(x)))dx.
\end{array}
\end{equation}

\subsection{Convexity shape prior method by Gorelick et al.}

Gorelick \emph{et al.} \cite{gorelick2014convexity} introduced a  convexity constraint based on graph theory. In their work, the convexity of objects can be represented as a sum of 3-clique potentials by penalizing any 1-0-1 configuration on all straight lines. Here the label 1 denotes the object, the label 0 denotes the background in an image, with details below.

Define a function $\Psi$ on all triplets of ordered pixels $(p,q,r)$ along all line $l\in L$, $L$ presents all the line on the image domain, as $\Psi: \{0,1\}\times\{0,1\}\times\{0,1\} \rightarrow \mathbb{R}$:
\begin{displaymath}
\Psi(u(p),u(q),u(r)) =
\left\{
\begin{array}{ll}
\infty & \text{if }(u(p),u(q),u(r)) = (1,0,1),\\
0&\text{otherwise},
\end{array}
\right.
\end{displaymath}
where $u(p),u(q),u(r)$ denote the label of point $p,q,r$, respectively, and $u(p)=1$ denotes that the label of point $p$ is object, $u(p) = 0$ denotes that the label of point $p$ is background.

Then, they define the convexity energy $E_{convexity}(u)$ as
\begin{displaymath}
E_{convexity}(u) = \sum_{l\in L} \sum_{(p,q,r)\in l} \Psi(u(p),u(q),u(r)).
\end{displaymath}

It is easy to see that the target region is convex if and only if $E_{convexity}(u) = 0$. By adding this energy to some proper chosen data term
$E_{data}(u) = \sum_{x\in\Omega} \left(-log P(I(x)|\theta^{u(x)})\right)$, where $\theta^{0},\theta^{1}$ are the parameters estimated from a bounding box given by the users. Here $P(\cdot)$ is the Gibbs prior distribution. Thus the total energy becomes
\begin{displaymath}
 E_{total}(u) = E_{convexity}(u) + E_{data}(u).
\end{displaymath}

Though this method can produce some approximate convex shape segmentations, it is not easy to extend it to general variational models for image segmentation.
In the following, we focus on transferring the geometrical relationship to functions relationships.


\section{Proposed Convex Shape Constraint}\label{sec:mainproof}
We will give some results with the help of implicit function presentation of a curve.
Our main result of convex shape prior is built upon the signed distanced function.
\begin{definition}[signed distance function, SDF \cite{chan2001active}]
Let  $\mathbb{D}$ be a subset of $\mathbb{R}^2$ with metric $d$ and $\partial \mathbb{D}$ be its boundary.
A function $\phi$:~~$\Omega\in \mathbb{R}^2 \rightarrow \mathbb{R}$ is called a signed distance function if:
\begin{displaymath}
\phi(x) =
\left\{
\begin{array}{ll}
-\underset{y\in \partial \mathbb{D}}{\inf} d(x,y)& \text{if } x\in \mathbb{D},\\
\underset{y\in \partial \mathbb{D}}{\inf} d(x,y)& \text{if } x\in \Omega\backslash \mathbb{D},
\end{array}
\right.
\end{displaymath}
where $\Omega$ is the image domain.
\end{definition}
In this paper, the metric $d$ is the commonly used Euclidean distance. To get a simple sufficient condition for convex region which is represented by SDF. In the following
analyses, we assume that the SDF $\phi\in C^2(\Omega), a.e.$.

In geometry, convex set can be defined as
\begin{definition}[Convex region] $\mathbb{D}$ is a convex region if and only if:
\begin{displaymath}
\forall x,y \in \mathbb{D}, \forall \lambda\in(0,1), \lambda x+(1-\lambda) y \in \mathbb{D}.
\end{displaymath}
\end{definition}

Using these two definitions and a mild smoothness assumption, we have some connections between signed distance function and convex region.
\begin{theorem}\label{convex:maintheorem}
Let $\phi$ be a signed distance function of a region $\mathbb{D}$.
 If $\phi \in C^2$ a.e. in $\Omega$ and
\begin{displaymath}
\Delta \phi (x) \geq 0, \forall x\in \Omega.
\end{displaymath}
Then we have that $\mathbb{D}$ must be a convex region.
\end{theorem}
Before we give the proof of Theorem \ref{convex:maintheorem},
let us give some useful propositions.

\begin{proposition}\label{prop1}
A region $\mathbb{D}$ is a convex set if the SDF $\phi$ of $\mathbb{D}$ is a convex function.
\end{proposition}

\begin{proof}
$\forall x,y \in \mathbb{D}, \forall \lambda\in(0,1)$.
Based on the definition of SDF,
\begin{displaymath}
\phi(x)<0,\phi(y)<0.
\end{displaymath}
Because $\phi$ is a convex function, then
\begin{displaymath}
\phi(\lambda x+(1-\lambda) y)\leq\lambda\phi(x) + (1-\lambda)\phi(y)<0.
\end{displaymath}
Due to the definition of SDF,
\begin{displaymath}
\lambda x+(1-\lambda) y\in \mathbb{D}.
\end{displaymath}
\end{proof}

\begin{proposition}\label{prop2}
For $\phi\in C^2(\Omega)$, $\phi$ is convex if the Hessian matrix of $\phi$ is positive semidefinite.
\end{proposition}
Let us denote $x=(x_1,x_2)$, then

\begin{proposition}\label{prop3}
For $\phi\in C^2(\Omega)$,  the Hessian matrix of $\phi$ is positive semidefinite if
\begin{displaymath}
\Delta \phi = \phi_{x_1x_1} + \phi_{x_2x_2} \geq0, \phi_{x_1x_1}\phi_{x_2x_2}-\phi_{x_1x_2}\phi_{x_2x_1} \geq0.
\end{displaymath}
\end{proposition}

\begin{proposition}\label{prop4}[Basic Property of SDF]
For any SDF $\phi$, its gradient satisfies the eikonal equation: $|\nabla \phi|=1,a.e.$.
\end{proposition}

\begin{proposition}\label{prop5}
If a SDF $\phi \in C^2$ a.e. in $\Omega$, then one can get  $\phi_{x_1x_1}\phi_{x_2x_2}-\phi_{x_1x_2}\phi_{x_2x_1} =0$ {a.e.} in $\Omega$.
\end{proposition}

\begin{proof}
From Proposition \ref{prop4}, we have
\begin{displaymath}
\phi_{x_1}^2 + \phi_{x_2}^2 =1, \emph{a.e.}.
\end{displaymath}
Differential by $x_1$ and $x_2$ separately, we have
\begin{displaymath}
\begin{split}
\phi_{x_1}\phi_{x_1x_1} + \phi_{x_2}\phi_{x_2x_1} = 0, \emph{a.e.}\\
\phi_{x_1}\phi_{x_1x_2} + \phi_{x_2}\phi_{x_2x_2} = 0, \emph{a.e.}.
\end{split}
\end{displaymath}
Transfer the second term to the right side, and multiply this two equation, we have
\begin{displaymath}
\phi_{x_1}\phi_{x_1x_1}\phi_{x_2}\phi_{x_2x_2} = \phi_{x_2}\phi_{x_2x_1}\phi_{x_1}\phi_{x_1x_2}, \emph{a.e.}.
\end{displaymath}

If $\phi_{x_1}\phi_{x_2}\neq 0$, we have $\phi_{x_1x_1}\phi_{x_2x_2}-\phi_{x_1x_2}\phi_{x_2x_1} =0, \emph{a.e.}$.

If $\phi_{x_1}=0$ or $\phi_{x_2} = 0$, then $\phi_{x_1x_1} = \phi_{x_1x_2}=0$ or $\phi_{x_2x_1} = \phi_{x_2x_2} = 0$, $\phi_{x_1x_1}*\phi_{x_2x_2}-\phi_{x_1x_2}*\phi_{x_2x_1} =0, \emph{a.e.}$ still holds.

Combine all above, we have $\phi_{x_1x_1}\phi_{x_2x_2}-\phi_{x_1x_2}\phi_{x_2x_1} =0$ \emph{a.e.}.
\end{proof}

Now we can prove the theorem \ref{convex:maintheorem}.
\begin{proof}
Using the definition of SDF and the proposition \ref{prop5}, we know that
$$\phi_{x_1x_1}\phi_{x_2x_2}-\phi_{x_1x_2}\phi_{x_2x_1} =0, a.e..$$
Combining with the assumption $\Delta \phi\geq 0$, and apply the proposition \ref{prop3}, we know that the Hessian matrix of $\phi$ is positive semidefinite at each point. By Proposition \ref{prop2}, we get $\phi$ is a convex function. Therefore, we get $\mathbb{D}$ is a convex region according to proposition \ref{prop1}.
\end{proof}

In the next, we shall propose some proper algorithms for variational level set methods with these constraints to guarantee convexity shape prior. The new models would be very simple but very efficient to keep convex shape prior.

\section{Apply to the proposed convexity shape prior method on Chan-Vese model}\label{sec:applyonDRLSE}

Inspired by \cite{estellers2012efficient,chan2001active},
we apply our method to Chan-Vese model and an edge based model in this section. Here we denote that
\[ \mathbb{C}_1 = \{\phi|~ |\nabla \phi| = 1\},\]
\[\mathbb{C}_2 = \{\phi|~ \Delta \phi \geq 0\}.\]
Then the Chan-Vese model with convexity shape prior would be
\begin{equation}\label{CVC}
\underset{c_1,c_2,\phi\in \mathbb{C}_1\cap\mathbb{C}_2}{\min}\left\{F(c_{1},c_{2},\phi)\right\}.
\end{equation}
The only difference between the original Chan-Vese model and (\ref{CVC}) is that there is a  constraint
$\phi\in\mathbb{C}_1\cap\mathbb{C}_2$ in (\ref{CVC}), and it can guarantee that $\phi$ must be a convex function. Then the segmented result must be a convex region according to theorem \ref{convex:maintheorem}.

\section{Numerical techniques for the constrained minimization problems}\label{sec:solveconvex}

Here, we build an algorithm for finding convex shape prior $\phi\in\mathbb{C}_1\bigcap\mathbb{C}_2$. For such a constraint,  penalty method, Lagrange multiplier method and augmented Lagrange multiplier methods \cite{Wu2010} \emph{etc.} can be applied. In this work, we would not apply these methods since they are relaxation methods. Namely, the constraint can not be strictly kept during the iterations. In fact, the constraint should be strictly kept at each step in order to fulfill the convexity shape prior in our method. We have tested the
penalty and Lagrange multiplier methods, they failed for our method in real implementations for holding the constraints.
The constraint set $\mathbb{C}_2$ is a convex set, and one possible optimiziation technique to to handle this constraints for minimize (\ref{CVC}) is to use a projection algorithm. However, such a projection does not have a closed-form solution and we have to propose an algorithm to fullfil this convex constraint almost everywhere in $\Omega$.

For $\phi\in\mathbb{C}_1\bigcap\mathbb{C}_2$,
we would like to split this problem into two subproblems which involve the constraints $\mathbb{C}_1$ and $\mathbb{C}_2$ seperately.

\subsection{Subproblem $\phi\in\mathbb{C}_1$}
This constraint is non-convex.
Since $\phi\in\mathbb{C}_1$ is very important for our method because of proposition (\ref{prop5}). Penalty method such as \cite{li2010distance} can not be used since we need to  keep the condition $|\nabla \phi|=1$ almost everywhere and numerically this constraint must be kept with good accuracy at each grid point.
Thus  we propose to use the fast matching method \cite{Sethian2000Level} to reinitialize the level set function.
It can be easily implemented by a Matlab function ``bwdist".
For better expression, we use $\mathcal{P}_1(\phi)$ to denote
this pseudo projection, i.e. $\mathcal{P}_1(\phi)\in\mathbb{C}_1$ is an optimal approximation of $\phi$ in the set $\mathbb{C}_1$.
We summarize this method in Algorithm \ref{convexalm_1}.

\begin{algorithm}
\caption{Algorithm for $\mathcal{P}_1(\phi)$.}
Input: $\phi$; output: $\mathcal{P}_1(\phi)$.
\begin{algorithmic}[1]
\label{convexalm_1}
\STATE  finding a zero contour of $\phi$, denoted by $\Gamma$.
\STATE  Using Matlab function ``bwdist" to get the distance between $x$ and $\Gamma$. Denote this distance as $d(x,\Gamma)$.
\STATE Get signed distance function $\mathcal{P}_1(\phi(x))$
by setting signs to $d(x,\Gamma)$.
\end{algorithmic}
\end{algorithm}

\subsection{Subproblem $\phi\in\mathbb{C}_2$}
It is quite easy to verify that the set $\mathbb{C}_2$ is a convex closed set in $C^{2}(\Omega)$, and the Laplace operator is a linear operator.
For a given function $\phi$, we shall consider the projection onto $\mathbb{C}_2$, namely
\begin{displaymath}\label{projlap}
\mathcal{P}_2(\phi) = \arg\min_{\psi \in \mathbb{C}_2} \frac{1}{2}||\phi - \psi||_{L^2(\Omega)}^2.
\end{displaymath}
By duality, this constrainted minimization problem have
the following dual energy

\begin{displaymath}
\min_{\psi}~\max_{\lambda\leqslant 0} \left\{L(\psi;\lambda)=\frac{1}{2}||\phi- \psi||_{L^2(\Omega)}^2 + \int_{\Omega}\lambda(x) \Delta \psi(x) dx\right\}.
\end{displaymath}
Assuming $(\psi^*,\lambda^*)$ is a saddle point of $L(\psi;\lambda)$, then following the well-known KKT conditions, we have
\begin{eqnarray}
\psi^*(x) - \phi(x) + \Delta \lambda^*(x) = 0, \forall x\in \Omega.\label{kkt1}\\
\Delta \psi^*(x)  \geq 0, \forall x\in \Omega.\label{kkt2}\\
\lambda^*(x) \Delta \psi^*(x) = 0,\forall x\in \Omega. \label{kkt3}\\
\lambda^*(x) \leq 0, \forall x\in \Omega.\label{kkt4}
\end{eqnarray}
Unfortunately, the closed-form solutions of $\lambda^*$ and $\psi^*$ are not available since there is a second order differentiable operator $\Delta$. However, we can use (\ref{kkt1}-\ref{kkt4}) to build an iteration method to find an approximation of $\lambda^*$ and $\psi^*$ using active set method.

Let us denote $\psi^m,\lambda^m$ be the $m$-th iteration of $\psi$ and $\lambda$. We would like to let $\psi^m,\lambda^m$ satisfy KKT conditions (\ref{kkt1}-\ref{kkt4}) in each iteration.

For the $m-1$-th iteration, denote the inactive set as $\mathbb{I}^{m-1} = \{x|~\Delta \psi^{m-1}(x)>0\}$ for a given $\psi^{m-1}$. The active set is $\widetilde{\mathbb{I}^{m-1}}= \{x|~\Delta\psi^{m-1}(x)\leq0\}$.
On the $m$-th iteration, we want to find a $\psi^{m}$ to strictly satisfy all the KKT conditions. First, we calculate the corresponding inactive set $\mathbb{I}^{m-1}$, and we get $\lambda^{m-1}(x) = 0$ for $x\in \mathbb{I}^{m-1}$ according to (\ref{kkt3}). Then we can set $\psi^{m} = \psi^{m-1}$ to satisfy (\ref{kkt1}). Otherwise, if $x\in \widetilde{\mathbb{I}^{m-1}}$, the constraint $\Delta \psi =0$ is active and we can solve the equation $\Delta \psi = 0$ to satisfies (\ref{kkt2}). Here, we can use Gauss-Seidel iteration to get $\psi^{m}$. Namely, $\psi^{m}(x) = \frac{1}{4}\sum_{y\in \mathbb{N}_4(x)}\psi^{m-1}(y)$. Here $\mathbb{N}_4(x)$ means the 4 neighborhood of point $x$.

Now under a fixed active set $\mathbb{I}^{m-1}$, we have an approximate  solution of our KKT condition.  Since the active set is only available if we have the real minimal, here we have to use an iteration method to find a solution that close to the real minimal. In the whole iteration, we set iteration number of the Gauss-Seidel method just to be one. If inactive set $\mathbb{I}^{m}$ and $\psi^{m}$ both converge, we get the final result $\mathcal{P}_2(\phi)= \psi^*$, where
$\underset{m\rightarrow +\infty}{\lim}\psi^m=\psi^*$.

We summary the Algorithm for calculating $\mathcal{P}_2(\phi)$ in Algorithm \ref{convexalm_2}.

\begin{algorithm}
\caption{Algorithm for calculating $\mathcal{P}_2(\phi)$.}
Input: $\phi$, output: $\mathcal{P}_2(\phi)$.
\begin{algorithmic}
\label{convexalm_2}
\STATE 1. Set $\psi^0 =\phi $. Finding the inactive set $\mathbb{I}^{0} = \{x|~\Delta\psi^{0}>0\}$. Let m=1.
\STATE 2. Updating $\psi^{m}$ by
$$\psi^{m}(x) =
\left\{
\begin{array}{ll}
\psi^{m-1}(x),&x\in \mathbb{I}^{m-1},\\
\frac{1}{4} \sum_{y\in \mathbb{N}_4(x)} \psi^{m-1}(y),& x\in \widetilde{\mathbb{I}^{m-1}}.
\end{array}
\right.$$
\STATE 3. Inactive set updating. Calculating $\mathbb{I}^{m}$ by the updated $\psi^m$.
\STATE 4. Checking if $\mathbb{I}^{m}=\mathbb{I}^{m-1}$.
 If this equation holds, then go to the next step and end the algorithm. Else, let $m = m+1$, and go back to step 2.
\STATE 5. Set $\mathcal{P}_2(\phi)= \psi^m$.
\end{algorithmic}
\end{algorithm}

\subsection{Summary of finding $\phi\in\mathbb{C}_1\cap\mathbb{C}_2$}
Here, we summarize the algorithms for finding $\phi\in\mathbb{C}_1\cap\mathbb{C}_2$, denote as $\mathcal{P}_3(\phi)\in \mathbb{C}_1\cap\mathbb{C}_2$.
\begin{algorithm}
\caption{Algorithm for calculating $\mathcal{P}_3(\phi)\in \mathbb{C}_1\cap\mathbb{C}_2$.}
Input: $\phi$, output: $\mathcal{P}_3(\phi)$.
\begin{algorithmic}
\label{convexalm_3}
\STATE 1. Set $\phi^0 =\phi$, $n=1$.
\STATE 2. Finding ${\phi}^{n}=(\mathcal{P}_2\circ\mathcal{P}_1)(\phi^{n-1})$ by Algorithm \ref{convexalm_1} and \ref{convexalm_2}.
\STATE 3. Checking convergence by error criterion such as
$\frac{||\phi^n-\phi^{n-1}||}{||\phi^n||}>\epsilon$. If it holds, let $n=n+1$ and go to the step 2. Else, go to the next step.
\STATE 4. Set $\mathcal{P}_3(\phi)= \phi^n$.
\end{algorithmic}
\end{algorithm}


\section{Numerical Algorithms for Chan-Vese model and edge based model with convex shape prior}
The following alternating minimization algorithm which has often be used in the literature for the Chan-Vese model can be used to solve the Chan-Vese convex shape prior minimization problem (\ref{CVC}):
\begin{equation*}
\left\{
\begin{array}{rcl}
\phi^{l+1}&=&\underset{\phi\in\mathbb{C}_1\bigcap\mathbb{C}_2}{\arg\min}~~F(c_1^{l},c_2^{l},\phi),\\
(c_1^{l+1},c_2^{l+1})&=&\underset{c_1,c_2}{\arg\min}~~F(c_1,c_2,\phi^{l+1}).
\end{array}
\right.
\end{equation*}

Combining with algorithm \ref{convexalm_3}, we get the following  iterative scheme for solving (\ref{CVC}):
\begin{displaymath}
\left\{
\begin{array}{rcl}
\phi^{l+\frac{1}{2}} &=& \phi^{l} - \Delta t \frac{\partial F}{\partial \phi}\Intvalx{\phi=\phi^{l},}~\\
\phi^{l+1} &=& \mathcal{P}_3(\phi^{l+\frac{1}{2}}),\\
c_{1}^{l+1}&=& \displaystyle\biggl.\frac{\int_{\Omega}I(x)H(\phi^{l+1}(x))dx}{\int_{\Omega}H(\phi^{l+1}(x))dx},\\
c_{2}^{l+1}&=&\displaystyle\biggl.\frac{\int_{\Omega}I(x)(1-H(\phi^{l+1}(x)))dx}{\int_{\Omega}(1-H(\phi^{l+1}(x)))dx},\\
l&=&0,1,2,3\dots.
\end{array}
\right.
\end{displaymath}
Here the operator $\mathcal{P}_3$ is
given by Algorithm \ref{convexalm_3}.

In our experiments, we use $\lambda_1= \lambda_2=1$ for (\ref{eq:DRLSE}). However,
for some natural images, the foreground and background are not homogeneous and thus we will only use the edge force to segment these images. For such cases, we will take $\lambda_1= \lambda_2=0$ in (\ref{eq:DRLSE}) and refer to this as the edge based model. Correspondingly, there is no need to compute and update the constants $c_i$ when edge based model is used.

\section{experiment results}\label{sec:exp}
Now, we demonstrate the segmentation results in this section. The computing platform for our numerical experiments is a laptop equipped with Intel Core i7 CPU @ 2.90GHz processor and 8G memory. The codes are written in Matlab R2013a without any special optimizing.

The Gaussian kernel used in our experiments is with size $5 \times  5$ and the standard deviation $0.5$.
For the color images, we extract the biggest contrast channel of the RGB color space, and treat them as gray images for segmentation.

\subsection{Testing of Algorithm \ref{convexalm_3}}
Firstly we test our Algorithm \ref{convexalm_3} on a synthetic non-convex region. We will show how the algorithm changes a non-convex region into a convex region. The size of test image is $200\times 200$.

\begin{figure}[!ht]
\begin{center}
\includegraphics{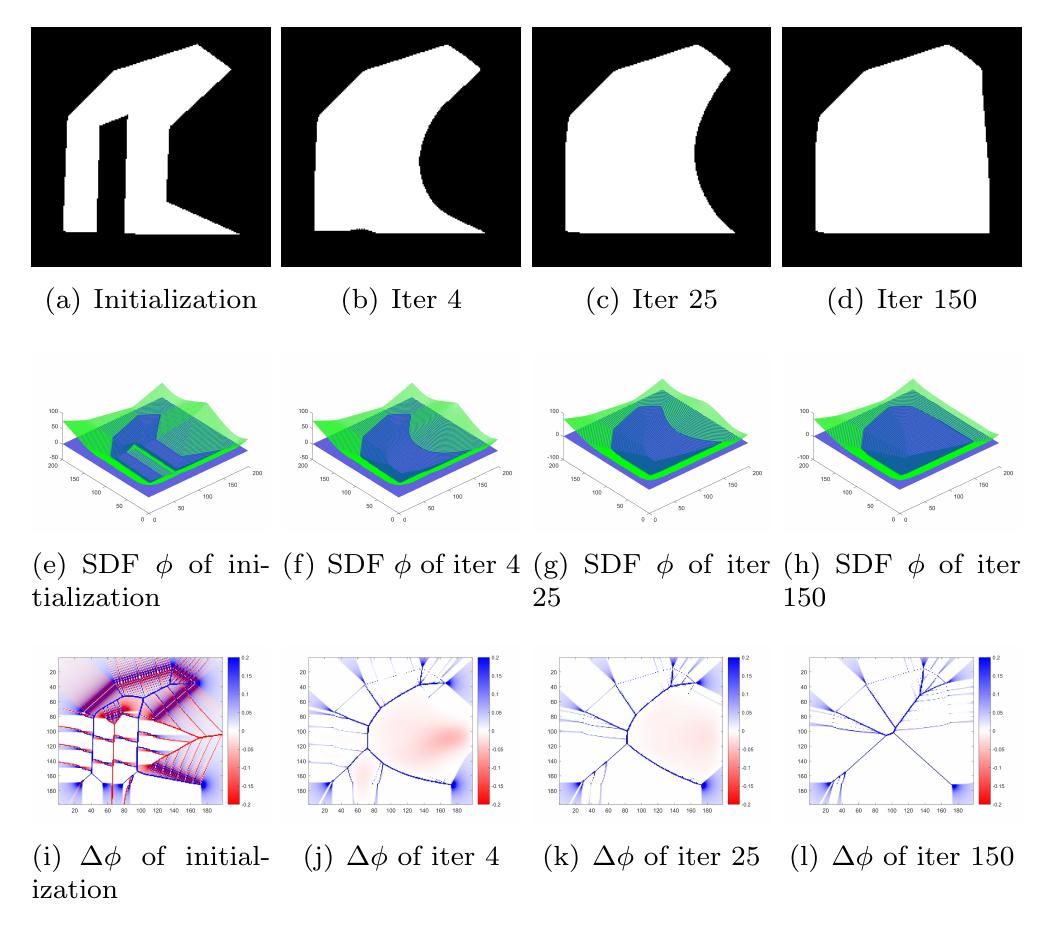}
\end{center}
\caption{Apply Algorithm \ref{convexalm_3} on a non-convex region, results from different iterations.}
\label{testonnonconvexregion}
\end{figure}

The test results are shown in Figure \ref{testonnonconvexregion}, where the images on the first row is the non-convex region and the results of iteration 4, 25, 150, respectively. Here this iteration number refers to the number $n$ in Algorithm \ref{convexalm_3}, the max iteration of m in Algorithm \ref{convexalm_2} is set to 30. The second row is the corresponding level set function in green, with the zero level set in blue.

To show that our method successfully keeps the constraint, we show the values of the Laplace of the level set function  in the third row. To be more clearly, we set the range of these values into [-0.2,0.2], as all we concern about is the relation with 0. For display, two value control dots are added in the upper left and down right corner. We use color blue to denote the positive value and red to denote the negative value. One can see from the images that the red zone becomes less and less and disappear at convergence. This means the Laplacian of the level set function is non-negative.

As you can see that the non-convex region is growing into a convex region near its convex hull. Near the part of the contour where the non-convex gap is wide (the right side), it needs more iterations to become convex. For the part of the contour where the gap is thin (the middle one), it turns to be convex in a few iteraions. This experiment shows that our method successfully turns a non-convex region into a convex region through its SDF using variational methods.

\subsection{Comparisons between Chan-Vese model, proposed convex prior model (region based), edge based model, proposed convex prior model (edge based) and \cite{gorelick2014convexity} on synthetic images}
In this section we apply these segmentation methods to the same images and compare the segmentation results. The images are taken from MPEG-7 CE-Shape-1 Dataset \url{https://cis.temple.edu/~latecki/}. The sizes of image are $256\times 256$. The parameters we used in this experiment are as follows: for the Chan-Vese model, $\mu=10$, as well as the proposed convex prior model (region based). 

Comparison results are listed in Figure \ref{figu1}, where the first column contains the original images, the second column to fifth column is the segmentation results of method Chan-Vese, proposed convex prior model (region based), edge based, proposed convex prior model (edge based), respectively. For better comparison, the manually label seeds of \cite{gorelick2014convexity} is given in the sixth column. Here the blue color represents object, and the red represents background. The seventh column is the results of \cite{gorelick2014convexity}. As you can see that our algorithm successfully segments the region as a convex region in the third and fifth columns. Compared to the semi-supervised method \cite{gorelick2014convexity} with discrete method, the results produced by the proposed are more smooth and tend to be more convex.
In addition, our algorithm does not need manually labeled seeds.

\begin{figure}[!h]
\begin{center}
\includegraphics{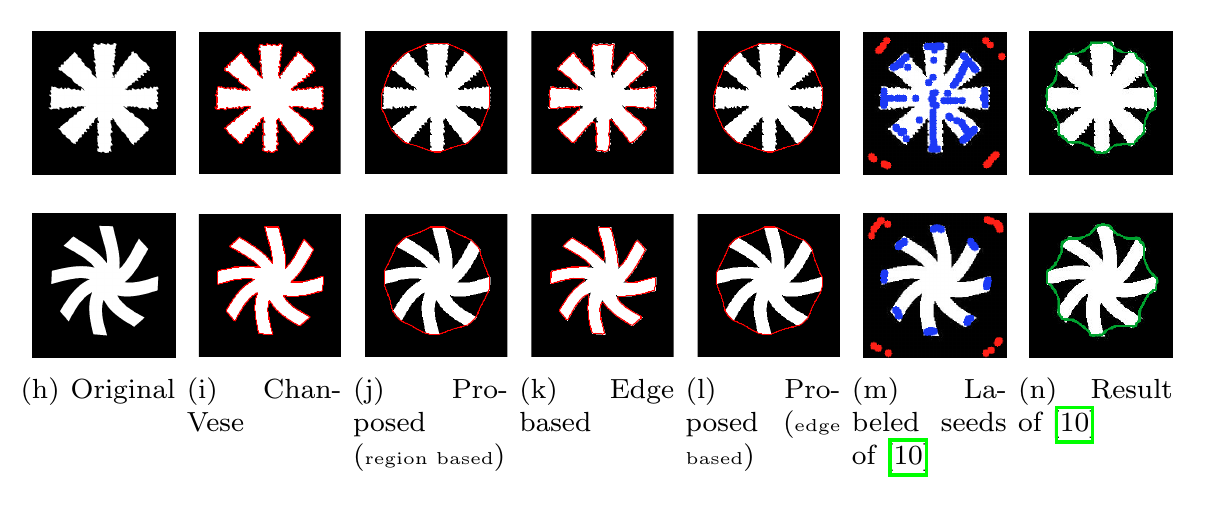}
\end{center}
\caption{The performance of convexity shape prior on different kind of shapes.}
\label{figu1}
\end{figure}

\subsection{Comparisons between models with and without convexity shape prior on real images}
Here, we compare the models with convexity shape prior and without convexity shape prior, of both Chan-Vese model and edge based model, on some nature images. The image is a box with some leafs on it, the size of the image is $128\times 128$, parameters appeared in the model are set as follows: for Chan-Vese model, $\mu=10$.
Note that our convexity shape prior is quite successful.

In Figure \ref{convex:figu5}, the first row left is the original image, the right is the image with initial curve with yellow color for all the experiments. The second row is the segmentation result of Chan-Vese model, proposed convex prior model (region based), edge based model, proposed convex prior model (edge based), respectively. The third and fourth row has the same setting as the third row in Figure \ref{testonnonconvexregion}.

\begin{figure*}[!ht]
\begin{center}
\includegraphics{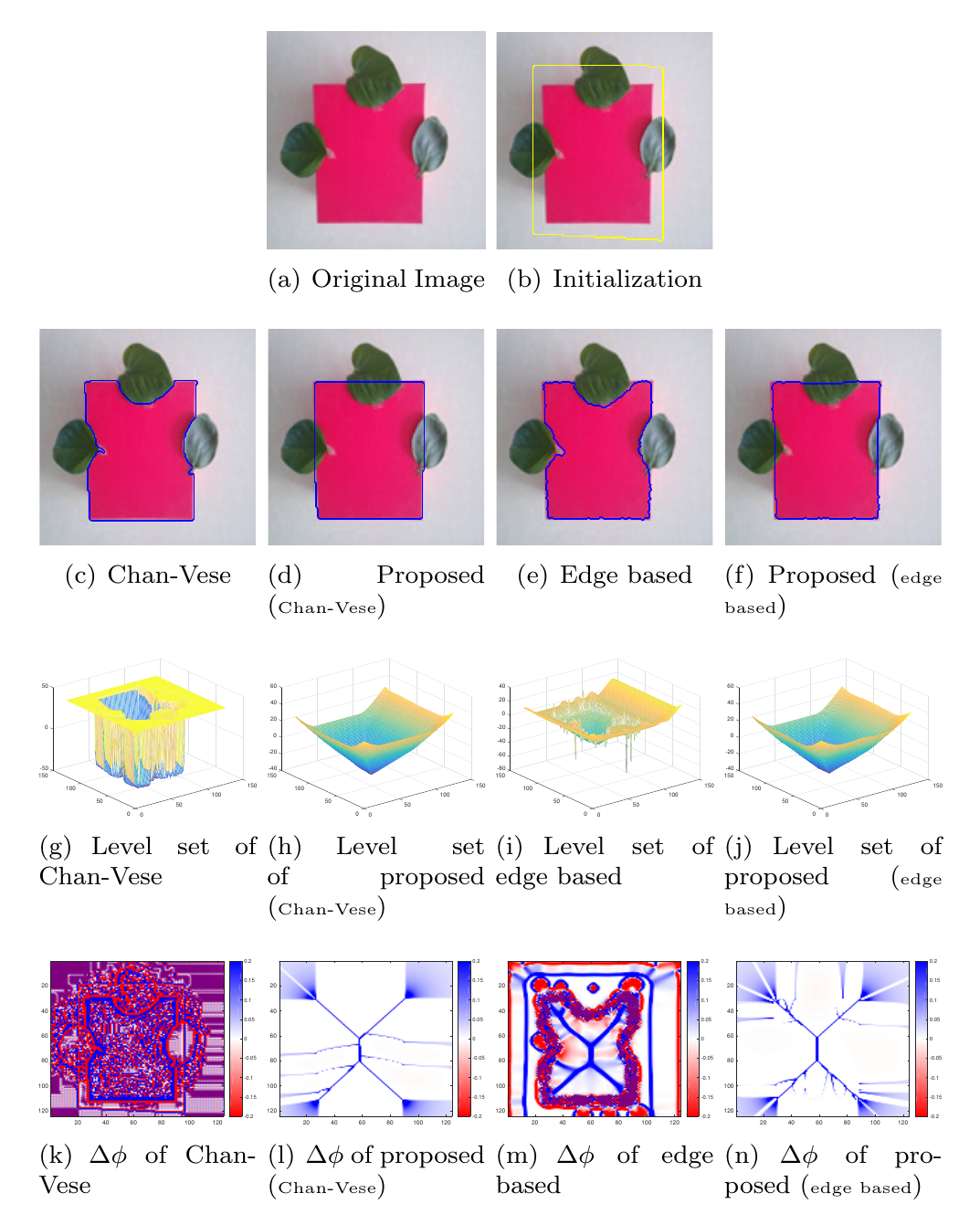}
\end{center}
\caption{Comparison between models with convexity shape prior and without convexity shape prior}
\label{convex:figu5}
\end{figure*}

\subsection{Initialization effects}
\begin{figure*}[!ht]
\begin{center}
\includegraphics{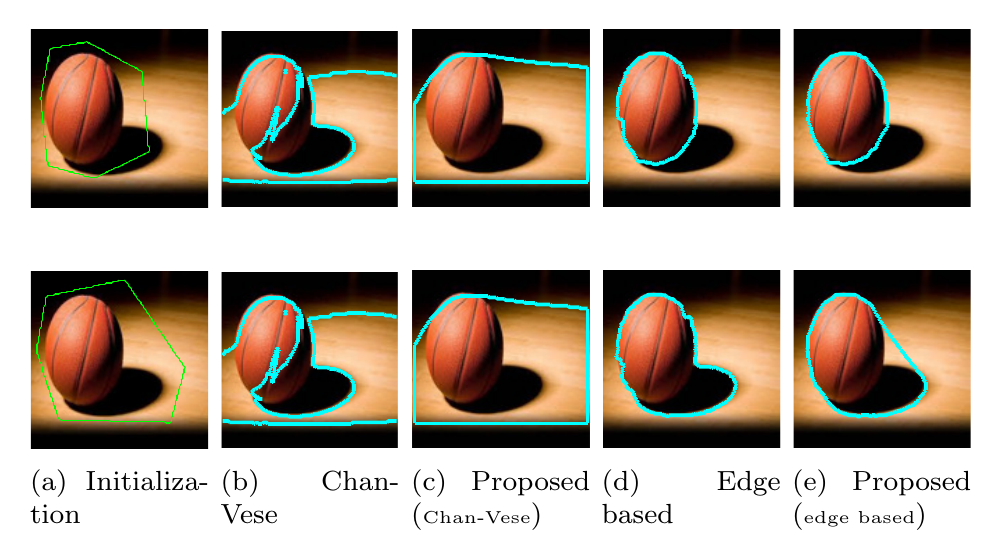}
\end{center}
\caption{Different initialization results on the same image.}
\label{compdiffini}
\end{figure*}
The solution found by the proposed convex prior model with the proposed algorithm might not be a global minimizer, but it must be a convexity region. So initialization is important for our algorithm. Different initializations may lead to different segmentation results. To be more comparable, we also add the segmentation results of Chan-Vese model.

This phenomena is shown and compared in Figure \ref{compdiffini}. The first column is the original image with the initialization of each method, presented as a green curve. In the following, second to fifth columns, we show the segmentation results of Chan-Vese model, proposed convex prior model (region based), edge based model, proposed convex prior model (edge based), respectively, the segmentation contours are displayed by cyan color. The size of image is $200\times 200$.

\subsection{Results on some natural images}

To further test our algorithms, we apply the proposed convex prior model (edge based) on more real images in Figure \ref{convex:figu4}, with the green curves as initial curves, the red curves as the segmentation results.
The image was resized to $200\times 200$.
\begin{figure*}[!ht]
\begin{center}
\includegraphics{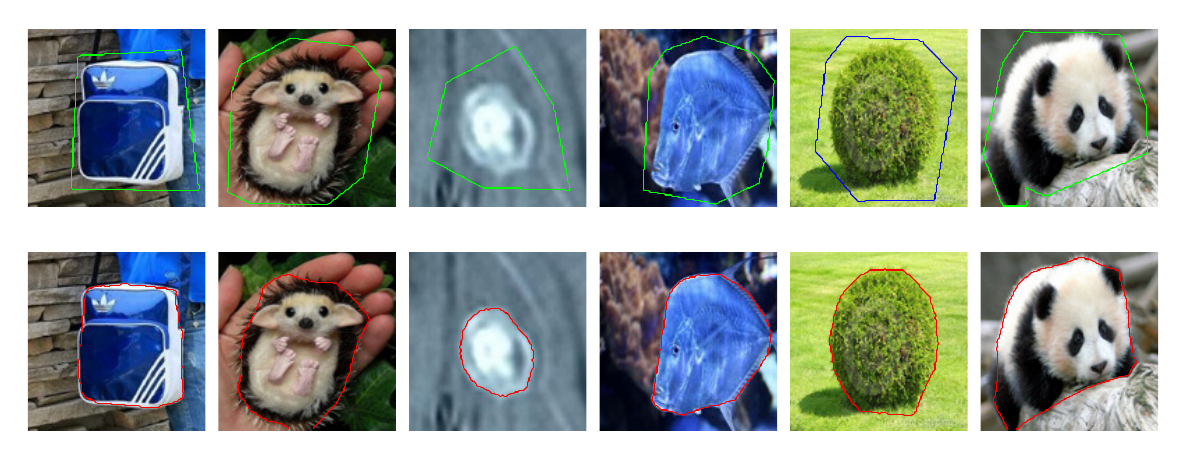}
\end{center}
\caption{Some results on real images.}
\label{convex:figu4}
\end{figure*}

\subsection{Test on images with occlusions}
\begin{figure*}[!ht]
\begin{center}
\includegraphics{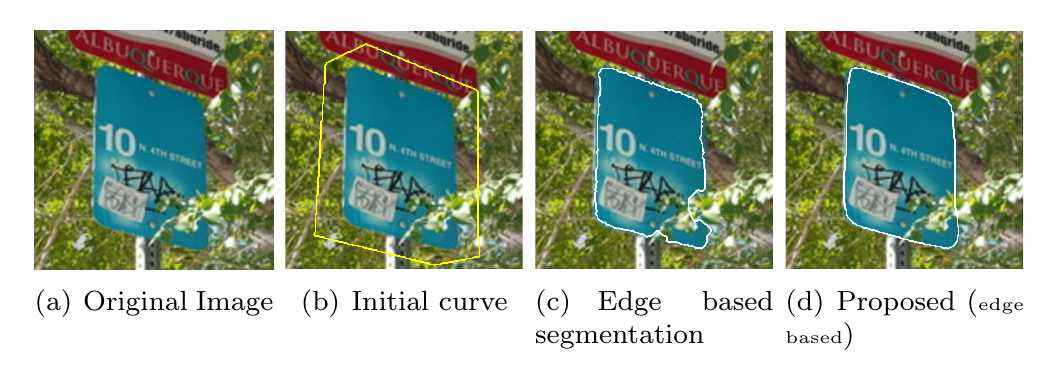}
\end{center}
\caption{Find the complete road sign.}
\label{convex:figu7}
\end{figure*}
In this part, we apply the proposed convex prior model (edge based) on a real image in Figure \ref{convex:figu7}. This image is a road sign, where the right down corner was covered by branches of a tree. The original image was from internet \url{https://seeclickfix.com/issues/1996027-abq-ride-please-move-bus-stop-sign-post-blocked-by-tree}, we are using part of it.
The image was resized to $200\times 200$.

In Figure \ref{convex:figu7}, (a) is the original image, (b) is the original image with initial curve, (c) is the segmentation result without the convex prior, (d) is the segmentation result of proposed model. As you can see, our model successfully segment the whole road sign.

\section{Conclusion and Discussion}\label{convex:conclusion}
In this paper, we proposed a convexity shape prior method for level set based segmentation method such as Chan-Vese model and edge based model. Through analyzing the connection of convex region and the corresponding signed distance function, we give out our convexity shape prior. By adding such a convex shape prior constraint into  these models, we get image segmentation models with convexity shape prior. Compared with the non-convex segmentation, we can see that our method successfully segmented the desired region as a convex region.
Although the algorithm works well and is very efficient, we do not yet know any convergence results for the proposed algorithm. In addition, the proposed method can not be directly extended to some region characteristic functions based segmentation models such as K-means, C-fuzzy, Potts model and GMM based EM segmentations. This is because the convex
property in theorem \ref{convex:maintheorem} would be lost in such methods without signed distance functions.
In the future, we may work on these aspects and extend our method to other applications such as medical images segmentations.
\section*{Acknowledgment}
\addcontentsline{toc}{section}{Acknowledgment}
Shi Yan was partly supported by the China Scholarship Council and University of Bergen.
Jun Liu and Haiyang Huang were partly supported by The National Key Research and Development Program of China (2017YFA0604903).
The authors would like to thank for the DRLSE codes provided by Chunming Li from website:
\url{http://www.imagecomputing.org/~cmli/DRLSE/}, and the codes provided by Lena Gorelick on her page \url{http://www.csd.uwo.ca/~ygorelic/publications.html}.

\bibliographystyle{bnubib}
\bibliography{bibtex}

\begin{thebibliography}{10}

\bibitem{leventon2000statistical}
M.E. Leventon, W.E.L. Grimson, O.~Faugeras.
\newblock Statistical shape influence in geodesic active contours[C].
\newblock Proceedings of Computer Vision and Pattern Recognition, 2000.
  Proceedings. IEEE Conference on, volume~1. IEEE, 2000.
\newblock  316--323.

\bibitem{rousson2002shape}
M.~Rousson, N.~Paragios.
\newblock Shape priors for level set representations[C].
\newblock Proceedings of Computer Vision, European Conference on. Springer,
  2002.
\newblock  78--92.

\bibitem{chan2005level}
T.F. Chan, W.~Zhu.
\newblock Level set based shape prior segmentation[C].
\newblock Proceedings of Computer Vision and Pattern Recognition, 2005. CVPR
  2005. IEEE Computer Society Conference on, volume~2. IEEE, 2005.
\newblock  1164--1170.

\bibitem{thiruvenkadam2007segmentation}
S.R. Thiruvenkadam, T.F. Chan, B.W. Hong.
\newblock Segmentation under occlusions using selective shape prior[C].
\newblock Proceedings of International Conference on Scale Space and
  Variational Methods in Computer Vision. Springer, 2007.
\newblock  191--202.

\bibitem{vu2008shape}
N.~Vu, B.S. Manjunath.
\newblock Shape prior segmentation of multiple objects with graph cuts[C].
\newblock Proceedings of Computer Vision and Pattern Recognition, 2008. CVPR
  2008. IEEE Conference on. IEEE, 2008.
\newblock  1--8.

\bibitem{boykov2001fast}
Y.~Boykov, O.~Veksler, R.~Zabih.
\newblock Fast approximate energy minimization via graph cuts[J].
\newblock IEEE Transactions on Pattern Analysis and Machine Intelligence, 2001,
  23(11):1222--1239.

\bibitem{kolmogorov2004energy}
V.~Kolmogorov, R.~Zabin.
\newblock What energy functions can be minimized via graph cuts?[J].
\newblock IEEE Transactions on Pattern Analysis and Machine Intelligence, 2004,
  26(2):147--159.

\bibitem{guo2015automatic}
W.~Guo, J.~Qin, S.~Tari.
\newblock Automatic prior shape selection for image segmentation[C].
\newblock Proceedings of Research in Shape Modeling. Springer, 2015: 1--8.

\bibitem{liu1999role}
Z.~Liu, D.W. Jacobs, R.Basri.
\newblock The role of convexity in perceptual completion: Beyond good
  continuation[J].
\newblock Vision Research, 1999, 39(25):4244--4257.

\bibitem{gorelick2014convexity}
L.~Gorelick, O.~Veksler, Y.~Boykov, et~al.
\newblock Convexity Shape Prior for Binary Segmentation[J].
\newblock IEEE Transactions on Pattern Analysis and Machine Intelligence, 2017,
  39(2):258--271.

\bibitem{royer2016convexity}
L.A. Royer, D.L. Richmond, C.~Rother, et~al.
\newblock Convexity Shape Constraints for Image Segmentation[C].
\newblock Proceedings of Proceedings of the IEEE Conference on Computer Vision
  and Pattern Recognition, 2016.
\newblock  402--410.

\bibitem{strekalovskiy2011generalized}
E.~Strekalovskiy, D.~Cremers.
\newblock Generalized ordering constraints for multilabel optimization[C].
\newblock Proceedings of Computer Vision (ICCV), 2011 IEEE International
  Conference on. IEEE, 2011.
\newblock  2619--2626.

\bibitem{bae2017augmented}
E.~Bae, X.C. Tai, W.~Zhu.
\newblock Augmented lagrangian method for an euler’s elastica based
  segmentation model that promotes convex contours[J].
\newblock Inverse Problems and Imaging, 2017, 11(1):1--23.

\bibitem{osher1988fronts}
S.~Osher, J.A. Sethian.
\newblock Fronts propagating with curvature-dependent speed: algorithms based
  on Hamilton-Jacobi formulations[J].
\newblock Journal of Computational Physics, 1988, 79(1):12--49.

\bibitem{kass1988snakes}
M.~Kass, A.~Witkin, D.~Terzopoulos.
\newblock Snakes: Active contour models[J].
\newblock International Journal of Computer Vision, 1988, 1(4):321--331.

\bibitem{Caselles1993}
V.~Caselles, F.~Catt{\'e}, T.~Coll, et~al.
\newblock A geometric model for active contours in image processing[J].
\newblock Numerische Mathematik, 1993, 66(1):1--31.

\bibitem{caselles1997geodesic}
V.~Caselles, R.~Kimmel, G.~Sapiro.
\newblock Geodesic active contours[J].
\newblock International Journal of Computer Vision, 1997, 22(1):61--79.

\bibitem{Jr1997A}
A.J. Yezzi, S.~Kichenassamy, A.~Kumar, et~al.
\newblock A geometric snake model for segmentation of medical imagery.[J].
\newblock IEEE Transactions on Medical Imaging, 1997, 16(2):199--209.

\bibitem{chan2001active}
T.F. Chan, L.A. Vese.
\newblock Active contours without edges[J].
\newblock IEEE Transactions on Image Processing, 2001, 10(2):266--277.

\bibitem{Lavdie2012A}
L.~Rada, K.~Chen.
\newblock A New Variational Model with Dual Level Set Functions for Selective
  Segmentation[J].
\newblock Communications in Computational Physics, 2012, 12(1):261--283.

\bibitem{Rada2013Improved}
L.~Rada, K.~Chen.
\newblock Improved selective segmentation model using one level-set[J].
\newblock Journal of Algorithms and Computational Technology, 2013,
  7(4):509--540.

\bibitem{Ali2017Multiphase}
H.~Ali, N.~Badshah, K.~Chen, et~al.
\newblock Multiphase segmentation based on new signed pressure force functions
  and one 1 level set function[J].
\newblock Turkish Journal of Electrical Engineering and Computer Sciences,
  2017, 25(4):2943--2955.

\bibitem{Sethian2000Level}
J.A. Sethian.
\newblock Level set methods and fast marching methods: evolving interfaces in
  computational geometry, fluid mechanics, computer vision, and materials
  science[M].
\newblock Cambridge University Press, 1999.

\bibitem{Dijkstra1959A}
E.~W. Dijkstra.
\newblock A note on two problems in connexion with graphs[J].
\newblock Numerische Mathematik, 1959, 1(1):269--271.

\bibitem{Zhao2005A}
H.K. Zhao.
\newblock A fast sweeping method for Eikonal equations[J].
\newblock Mathematics of Computation, 2005, 74(250):603--627.

\bibitem{li2010distance}
C.~Li, C.~Xu, C.~Gui, et~al.
\newblock Distance regularized level set evolution and its application to image
  segmentation[J].
\newblock IEEE Transactions on Image Processing, 2010, 19(12):3243--3254.

\bibitem{6944678}
Y.~Liu, C.~Li, S.~Guo, et~al.
\newblock A novel level set method for segmentation of left and right
  ventricles from cardiac MR images[C].
\newblock Proceedings of 2014 36th Annual International Conference of the IEEE
  Engineering in Medicine and Biology Society, 2014.
\newblock  4719--4722.

\bibitem{10.1007/978-3-642-40811-3_60}
C.~Feng, C.~Li, D.~Zhao, et~al.
\newblock Segmentation of the Left Ventricle Using Distance Regularized
  Two-Layer Level Set Approach[C].
\newblock In: Kensaku Mori, Ichiro Sakuma, Yoshinobu Sato, et~al., (eds.).
  Proceedings of Medical Image Computing and Computer-Assisted Intervention --
  MICCAI 2013, Berlin, Heidelberg: Springer Berlin Heidelberg, 2013.
\newblock  477--484.

\bibitem{doi:10.1118/1.4947126}
C.~Feng, S.~Zhang, D.~Zhao, et~al.
\newblock Simultaneous extraction of endocardial and epicardial contours of the
  left ventricle by distance regularized level sets[J].
\newblock Medical Physics, 43(6Part1):2741--2755.

\bibitem{estellers2012efficient}
V.~Estellers, D.~Zosso, R.~Lai, et~al.
\newblock Efficient algorithm for level set method preserving distance
  function[J].
\newblock IEEE Transactions on Image Processing, 2012, 21(12):4722--4734.

\bibitem{Wu2010}
C.~Wu, X.C. Tai.
\newblock Augmented Lagrangian method, dual methods, and split Bregman
  iteration for ROF, vectorial TV, and high order models[J].
\newblock SIAM Journal on Imaging Sciences, 2010, 3(3):300--339.

\end{thebibliography}

\end{document}